\documentclass[runningheads]{llncs}
\usepackage{graphicx}
\usepackage{amsmath}
\usepackage{amssymb}
\usepackage{amsfonts}
\usepackage{todonotes}

\usepackage[hidelinks]{hyperref}

\newcommand{\norm}[1]{\left\lVert#1\right\rVert}
\def\abs#1{\left\lvert#1\right\rvert}
\begin{document}
\title{A computational geometry approach for modeling neuronal fiber pathways\texorpdfstring{\thanks{Supported by NSF award: SSI \# 1664172 and NIH award \# 5R01NS103774-02.}}{}}
\titlerunning{A computational geometry approach for modeling neuronal fiber pathways}
%

%


\author{S. Shailja\orcidID{0000-0002-5056-9989}\and
Angela Zhang\orcidID{0000-0003-4198-8927} \and
B.S. Manjunath\orcidID{0000-0003-2804-3611}}
\index{{Shailja, S.} \and {Zhang, Angela} \and {Manjunath, B.S.}}

\authorrunning{S. Shailja et al.}
\institute{University of California, Santa Barbara CA 93117, USA\\
\email{\{shailja, angela00, manj\}@ucsb.edu}}


%
\maketitle              

\begin{abstract}
We propose a novel and efficient algorithm to model high-level topological structures of neuronal fibers. Tractography constructs complex neuronal fibers in three dimensions that exhibit the geometry of white matter pathways in the brain. However, most tractography analysis methods are time consuming and intractable. We develop a computational geometry-based tractography representation that aims to simplify the connectivity of white matter fibers. Given the trajectories of neuronal fiber pathways, we model the evolution of trajectories that encodes geometrically significant events and calculate their point correspondence in the 3D brain space. Trajectory inter-distance is used as a parameter to control the granularity of the model that allows local or global representation of the tractogram. Using diffusion MRI data from Alzheimer’s patient study, we extract tractography features from our model for distinguishing the Alzheimer's subject from the normal control. Software implementation of our algorithm is available on GitHub.
\end{abstract}
\section{Introduction}
Diffusion MRI (dMRI) tractography~\cite{basser2000vivo} constructs morphological 3D neuronal fibers represented by 3D images called tractograms. In recent years, analysis of fibers in dMRI tractography data has received wide interest due to its potential applications in computational pathology, surgery, and studies of diseases, such as brain tumors~\cite{berman2004diffusion,kao2020corrigendum}, Alzheimer's~\cite{bozzali2002white}, and schizophrenia~\cite{park2004white}. 
Tractography datasets are huge and complex consisting of millions of fibers arising and terminating at different functional regions of the brain. Computational analysis of these fibers is challenging owing to their complex topological structures in three dimensions. Tractography produces white matter pathways that can be deduced as spatial trajectories represented by a sequence of 3D coordinates. To model the geometry of these trajectories, we utilize the concept of Reeb graphs~\cite{shinagawa1991surface} that have been successfully used in a wide variety of applications in computational geometry and graphics, such as shape matching, topological data analysis, simplification, and segmentation. We assume that the groups of trajectories that are spatially close to each other share similar properties. Therefore, we compute a model to encode the arising \& ending and the merging \& splitting behavior for groups of trajectories (as shown in Fig.~\ref{fig:ex1}) along with their point correspondence. With these computations in place, we develop a finite state machine that can be used to query the state of any trajectory or its shared groups. The resulting model has tunable granularity that can be used to derive models with the desired level of geometrical details or abstract properties. 
\section{Related Work}
Brain tractography datasets are constructed from the dMRI of an individual's brain~\cite{Yeh2011,Yeh2013}. One way to analyze the fiber tracts is to generate a connectivity matrix that provides a compact description of pairwise connectivity of regions of interest (ROI) derived from anatomical or computational brain atlases. For example, the connectivity matrices can be used to compute multiple graph theory-based metrics to distinguish between the brains of healthy children and those with recent traumatic brain injury~\cite{Watson2019}. However, such methods overlook the geometrical characteristics within a region of interest. A number of inter-fiber distance-based approaches have been used to analyze the fibers~\cite{andersson2008reporting,brun2004clustering,dodero2015automated,o2015statistical,moberts2005evaluation,zhang2008identifying} for clustering and segmentation but have some limitations. For example, one needs prior information about the number of clusters to be segmented in~\cite{dodero2015automated}. More sophisticated methods produce high-dimensional representations that are not efficient~\cite{wang2011fiber,yendiki2011automated}. Due to the complex nature of tractography algorithms, another way to compare bundles is by using tract profiling techniques that quantifies diffusion measures along each fiber tract~\cite{yeatman2012tract}. Notably, researchers in~\cite{cabeen2020tractography} introduced a representation that is sparse and can be integrated with learning methods for further study. However, their approach leads to possible loss of critical points of fibers (due to polynomial fitting) and ignores multi-fiber tractography. Our design addresses this by sequentially processing the group behavior emerging due to events of individual trajectories. Our method builds on previous work on time-dependent trajectory analysis using a Reeb graph. A deterministic algorithm for Reeb graph computation in $O(n\log n)$ time is shown in~\cite{parsa2013deterministic}. Reeb graph can be used to model the trajectory grouping structure defined by time as a parameter~\cite{buchin2013trajectory}. For tractography analysis, the concept of ``bundling" and ``unbundling" structure of trajectory data to compute a sparse graph is proposed in~\cite{sun2015reeb}. They show graph representation of brain tractography but do not present the algorithm or proofs for the computation, focusing instead on the novel problem definition. 
\section{Preliminaries}\label{sec:def}
In the three dimensional Euclidean space $\mathbb{R}^3$, we define the following terms that would help in setting up the problem in this section.


\textbf{\textit{Trajectory}:} A trajectory $T$ is as an ordered sequence of points in $\mathbb{R}^3$. We denote a trajectory $T$ as a sequence of points $\{p_1, p_2, ... , p_m\}$ , where $m$ is the number of points in $T$ and $p_i \in \mathbb{R}^3$.
\begin{figure}[t!]
\scriptsize
            \centering
            \centerline{\includegraphics[scale = 0.25]{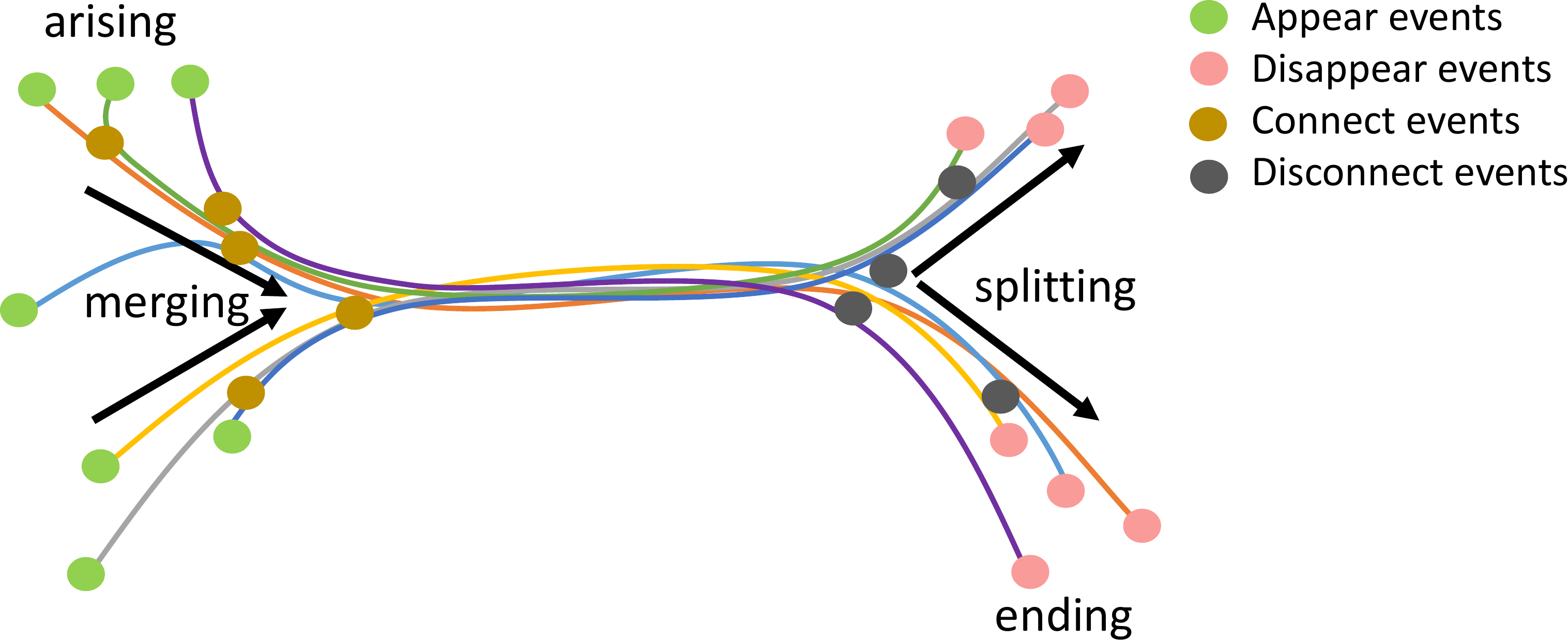}}
        \caption{A basic example of a set of trajectories displaying the arising, merging, splitting, and ending behaviour. These qualitative behaviors of group of trajectories emerge due to the events (appear, connect, disconnect, and disappear) of individual trajectory. Events of trajectories are used to define behavior of the group of trajectories.}
        \label{fig:ex1}

    \end{figure}

\textbf{\textit{$\epsilon$-(dis)connected points}}: For any pair of points $p_1$ and $p_2$ in $\mathbb{R}^3$, we define $d(p_1, p_2)$ as the Euclidean distance between the two points:
\begin{equation*}
    d(p_1, p_2) =  \norm{p_1 - p_2}_2,
\end{equation*}
where $\norm{\cdot}_2$ represents the Euclidean norm.
Two points $p_1, p_2$ are $\epsilon$-connected if $d(p_1, p_2) \leq \epsilon$. Similarly, two points $p_1, p_2$ are $\epsilon$-disconnected if $d(p_1, p_2) > \epsilon$.

\textbf{\textit{Appear event:}} For each trajectory $T$, the initial point of its ordered sequence is labeled for the occurrence of the \textit{appear event}. For example, for trajectory $T_1 = \{p_1, p_2, ... , p_m\}$, we observe the appear event at the point $p_1$. 

\textbf{\textit{Disappear event:}} For each trajectory $T$, the final point of its ordered sequence is labeled for the occurrence of the \textit{disappear event}. For example, for trajectory $T_1 = \{p_1, p_2, ... , p_{m}\}$ , we observe the disappear event at $p_{m}$. 

\textbf{\textit{Connect events:}} To define \textit{connect events} for a pair of trajectories, consider two trajectories
\begin{align*}
    T = \{p_1, p_2, ..., p_m\}, \quad
    T^{'} = \{p^{'}_1, p^{'}_2, ..., p^{'}_m\},
\end{align*}
then a connect event for the pair $(T, T^{'})$ is defined by $(p_i, p_j^{'})$ such that $p_i \in T, p_j^{'} \in T^{'}$ and,
\begin{align*}
    d(p_i, p_j^{'}) \leq \epsilon, \quad
    d(p_{i-1}, p_{j-1}^{'}) > \epsilon, \quad 
    \text{ for } i > 1 \text{ and } j >1.
\end{align*}
If there is no such pair of points, it implies that $T$ and $T^{'}$ are disjoint. Moreover, if $T$ and $T^{'}$ are $\epsilon$-connected at $(p_i, p_j^{'})$ and if $T^{'}$ and $T^{*}$ are also $\epsilon$-connected at $(p_j^{'}, p_l^{*})$ where $p_l^{*} \in T^{*}$, then we say that $T$ and $T^*$ are \textit{$\epsilon$-step connected} at $(p_i, p_l^{*})$.

\textbf{\textit{Disconnect events:}} Given a pair of trajectories ($T$, $T^{'}$) with a connect event at $(p_i, p_j^{'})$, we define a \textit{disconnect event} by $(p_{i+k}, p^{'}_{j+k})$ such that,
\begin{align*}
    d(p_{i+k}, p_{j+k}^{'}) > \epsilon, \quad
    d(p_{i+k-1}, p_{j+k-1}^{'}) \leq \epsilon.
\end{align*}

\textbf{\textit{Max-width $\epsilon$-connected trajectories:}} For an input $\mathcal{I}$, there are many possible $\epsilon$-step connected trajectories. 
The maximal group of trajectories at a given step $k$ are called \textit{max-width $\epsilon$-step connected} and there is no other possible set of sub-trajectories that can intersect with the maximal group at $k$.\\

Note that the trajectories estimated from dMRI tractography do not have a specific beginning or ending, as dMRI is not sensitive to the direction of connections. So, reversing the order of the points of a streamlines will produce similar results. Two events appear and disappear are used for convenience in describing the algorithm and its implementation.
\begin{figure}[t]
\scriptsize
            \centering
            \centerline{\includegraphics[scale = 0.23]{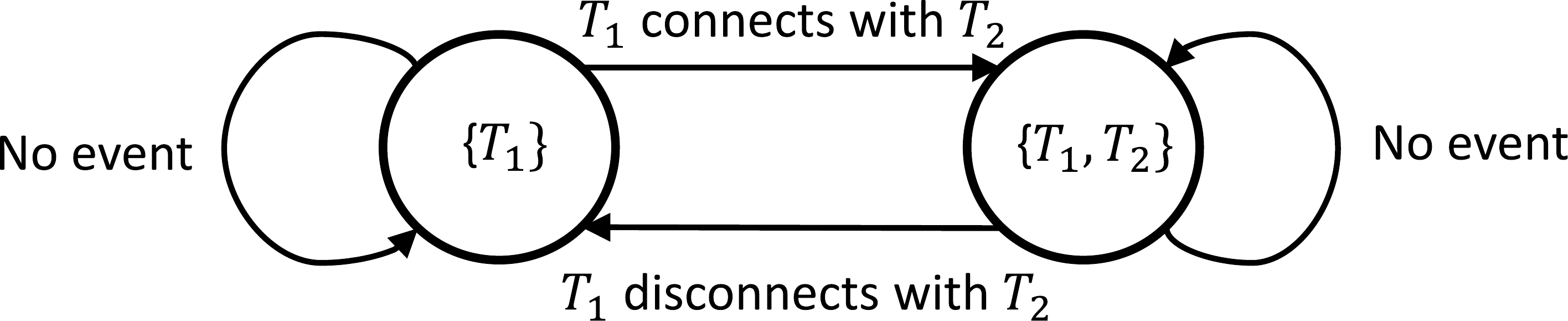}}
        \caption{An example of state diagram for trajectory $T_1$. $T_1$ is either directly (dis)connected with $T_2$ or $\epsilon$-step (dis)connected.}
        \label{fig:fsm}

    \end{figure}

\subsection{Problem Formulation}
We set up the following central problem for this paper:

\textit{Input:} 
A set of trajectories $\mathcal{I} = \{T_1, T_2, ... , T_n\}$, such that $T_i \in \mathbb{R}^3$ for all $i \in \{1,2,...,n\}$ where $n$ is the number of trajectories.


\textit{Output:} A finite state machine (FSM) $S$ that models the evolution of trajectories and their critical points of interaction with all other trajectories. 
\begin{align*}
    S &= (A, O, V, \mathcal{R}),\\
    \mathcal{R}&: (S \times A) \longrightarrow (S \times O),
\end{align*}
where $A$ is the set of events associated with each trajectory $T_{i} \in \mathcal{I}$, $O$ is the set of outputs encoding the location information of the critical events, $V$ is the set of states that corresponds to a group of trajectories. $\mathcal{R}$ is the state-transition and output function. When the machine is in a current state $v \in V$ and receives an input $a \in A$ it moves to the next state specified by $R$ and produces an output location $o \in O$ as shown in Fig.~\ref{fig:fsm}.




\section{Reeb Graph}
The central part of solving the problem as stated above is to compute $\mathcal{R}$ --- the state transition and output function. Towards that end, we compute an undirected graph $\mathcal{R}$ called the Reeb graph. In this section, we define the Reeb graph and then proceed to develop an algorithm that can compute this graph for a set of trajectories. Formally, a Reeb graph $\mathcal{R}$ is defined on a manifold $\mathcal{M} \in \mathbb{R}^3$ using the evolution of level sets $L$~\cite{doraiswamy2009efficient}. To adapt this definition of $\mathcal{R}$ for the case of neuronal fiber trajectory evolution problem, we define a manifold $\mathcal{M}$ in $\mathbb{R}^3$ as the union of all points in the tractogram. The set of points of trajectories at step $k$ is the level set of k. The connected components in the level set of $k$ correspond to the max-width $\epsilon$-connected trajectories at step $k$. Unlike previous studies~\cite{buchin2013trajectory}, here, any number of trajectories can become $\epsilon$-(dis)connected at the same location. Reeb graph $\mathcal{R}$ describes the evolution of the connected components over sequential steps. At every step $k$, the changes in connected components (states of FSM) are represented by vertices in $\mathcal{R}$.
\subsection{Computing the Reeb Graph}
\begin{figure}[t]
\scriptsize
            \centering
            \centerline{\includegraphics[scale = 0.25]{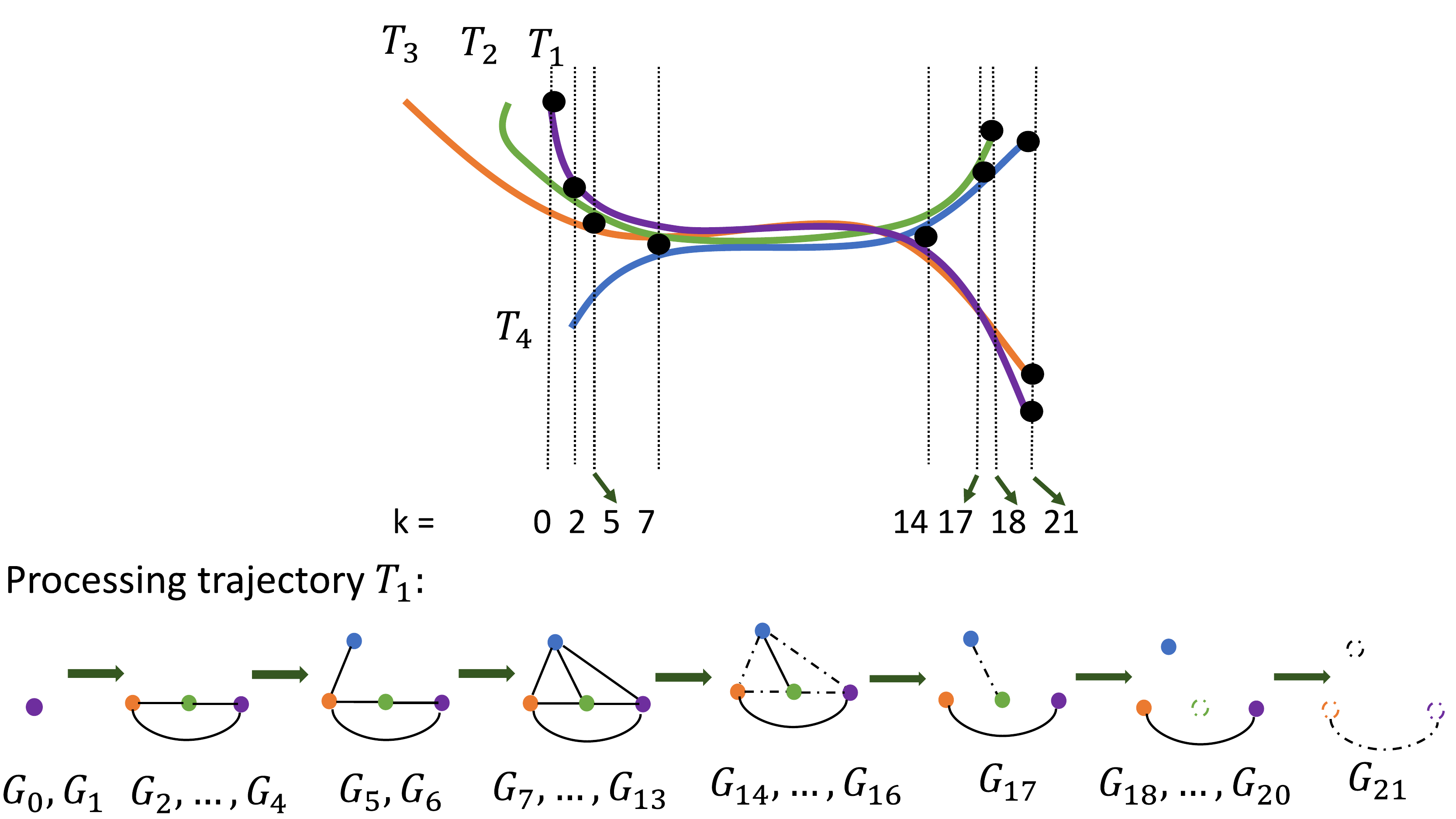}}
        \caption{The figure shows an input $\mathcal{I} = \{T_1, T_2, T_3, T_4\}$ exhibiting arising, merging, splitting, and ending behaviour. After processing the sequence of points in $T_1$, at $k = 0, 2, 5, 7, 14, 17, 18, 21$ steps, we modify the $G_k$ respectively. Connect, disconnect, appear, and disappear events associated with $T_1$ are marked by black circles. Delete node and edge queries are represented by dashed circle and dashed line in $G_k$.}
        \label{fig:ex_G}
    \end{figure}
In Section~\ref{sec:def}, for a given trajectory, we defined appear and disappear events. For a pair of trajectories, we defined connect and disconnect events. These events ($a \in A$) describe the branching structure of the trajectories. To compute the Reeb graph, we process these events sequentially. We maintain a graph $G = (V^{'}, E^{'})$ where the vertices represent the set of trajectories. $G$ is a graph that changes with steps representing the connect and disconnect relations between different trajectories. At each step $k$, we insert new nodes at appear events and delete nodes at disappear events. At connect events, we insert edges in $G$ and at disconnect events, we delete edges. At each step $k$, an edge $(T_1, T_2$) in $G$ shows that $T_1$ and $T_2$ are directly connected. Therefore the max-width $\epsilon$-connected trajectories correspond to the connected components in $G$ at step $k$.

\textbf{Initialization:} We spend $O(N^{2})$ time to store the appear, disappear, connect, and disconnect events for all pair of the trajectories. We store a mapping $M$ from the current components in $G$ to the corresponding vertices in $\mathcal{R}$. We start from one of the trajectories and add other trajectories of interest on the way of following its points sequentially. We maintain a data structure to flag the points for which the events are already processed and store their mappings to the vertices of the Reeb graph in $D$. Note that although the computational time is $O(N^{2})$, this step is massively parallelizable.  
\begin{figure}[t]
\scriptsize
            \centering
            \centerline{\includegraphics[scale = 0.23]{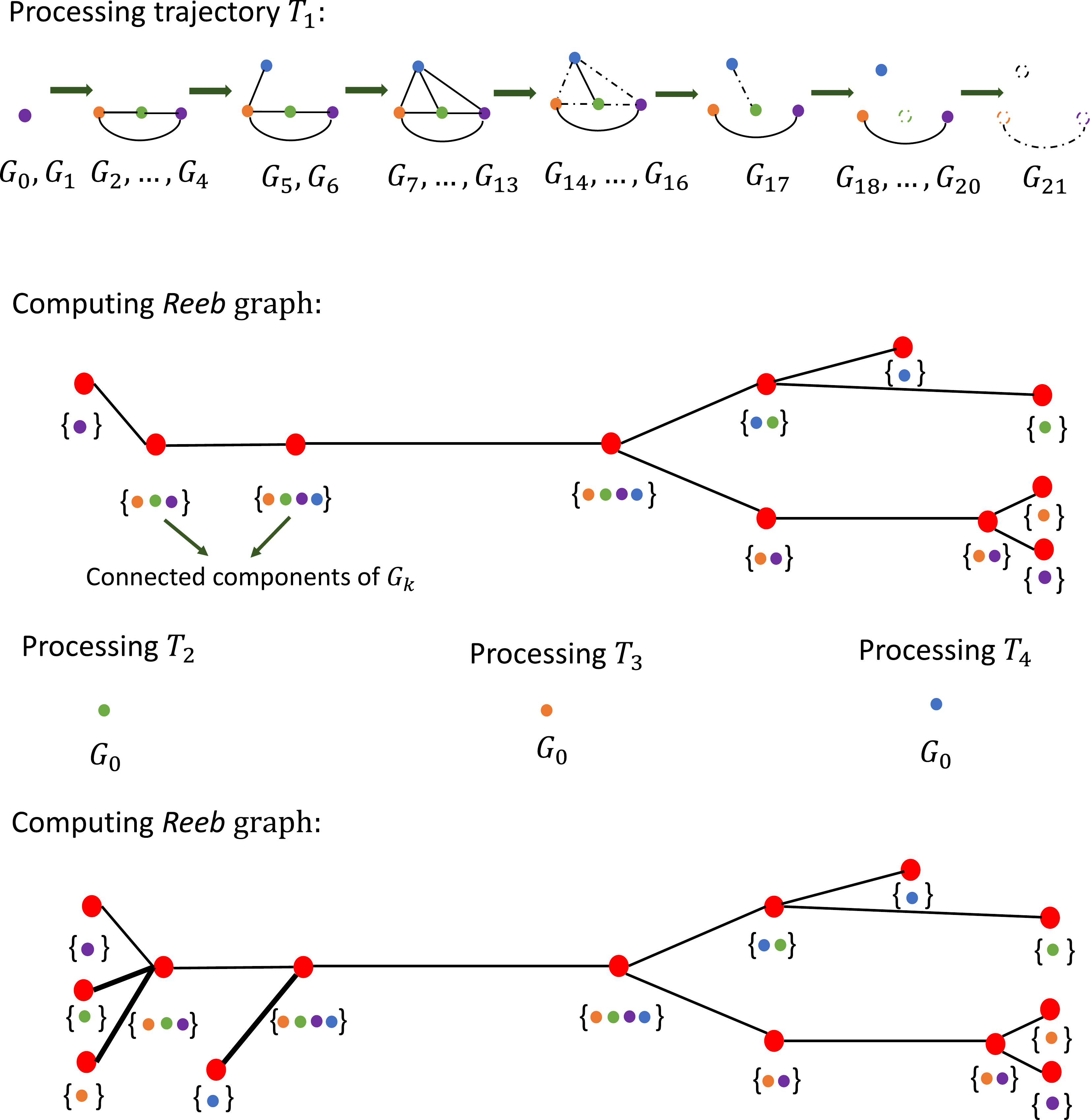}}
        \caption{Continuing the same example from Fig.~\ref{fig:ex_G}, we show representation of $\mathcal{R}$ from $G_k$. Edges of $\mathcal{R}$ encodes the maximal group of trajectories and vertices of $\mathcal{R}$ records the significant points of the trajectories.}
        \label{fig:ex_Gall}

    \end{figure}

\textbf{Split and Merge:} To handle a disconnect event of trajectories $T_1$ and $T_2$ at step $k$, we delete the edge $(T_1, T_2)$ from $G_{k}$. Similarly, for a connect event of trajectories $T_1$ and $T_2$ at step $k$, we add the edge $(T_1, T_2)$ to $G_{k}$. We do this for all the connect and disconnect events as shown in Fig.~\ref{fig:ex_G} for trajectory $T_1$. For the disconnect event, we query $G_{k-1}$ to get the connected component $C$ consisting of trajectories $T_1$ and $T_2$ and locate the corresponding $u$ in $\mathcal{R}$. We query $G_{k}$ to get the connected components $C_1$ and $C_2$ consisting of trajectories $T_1$ and $T_2$, respectively. $C_{1} = C_{2}$ implies that the trajectories $T_1$ and $T_2$ are still $\epsilon$-step connected. If  $C_{1} \neq C_{2}$, we add a new split vertex $v$ to $\mathcal{R}$ and a new edge $(u, v)$ and update $M$ accordingly.

\textbf{{Computing $\mathcal{R}$ from $G$:}} We query $G_{k}$ and $G_{k-1}$ to get the connected components at step $k$ and $k-1$ respectively. For each connected component $C_{c}$ in $G_{k}$, if $C_{c}$ is present in the connected components of $G_{k-1}$, then we do not modify $\mathcal{R}$. This implies that no such event occurred in the trajectories of $C_{c}$ which could result in any critical points. Otherwise, using $M$, we locate the corresponding nodes in $\mathcal{R}$ for the connected components in $G_{k-1}$, we call it previous connected components. The corresponding nodes in $\mathcal{R}$ for the connected components in $G_{k}$ are called present connected components. For each component in present connected components, we add a node $v$ in $\mathcal{R}$ if not already present in the previous connected component and assign the location ($o \in O$) as the coordinates of one of the points in the connected components. If that is the case, we also add an edge $(u, v)$, where $u$ is the node corresponding to previous connected component $C_1$ and $v$ is the node corresponding to present connected component $C_{2}$, if $|C_1 \cap C_2| > 0$. Finally, we update $M$ accordingly.

At next step $k + 1$, if we encounter the point of a trajectory $T$ for which the events have already been processed, we query $D$ to locate the vertex $u$ and $v$ in $\mathcal{R}$ for $p_{k}$ and $p_{k+1}$ respectively. We add an edge $(u, v)$ to $\mathcal{R}$ as shown in Fig.~\ref{fig:ex_Gall} and delete the node corresponding to trajectory $T$ in $G$ and update $M$.
\begin{figure}[t]
\scriptsize
            \centering
            \centerline{\includegraphics[scale = 0.17]{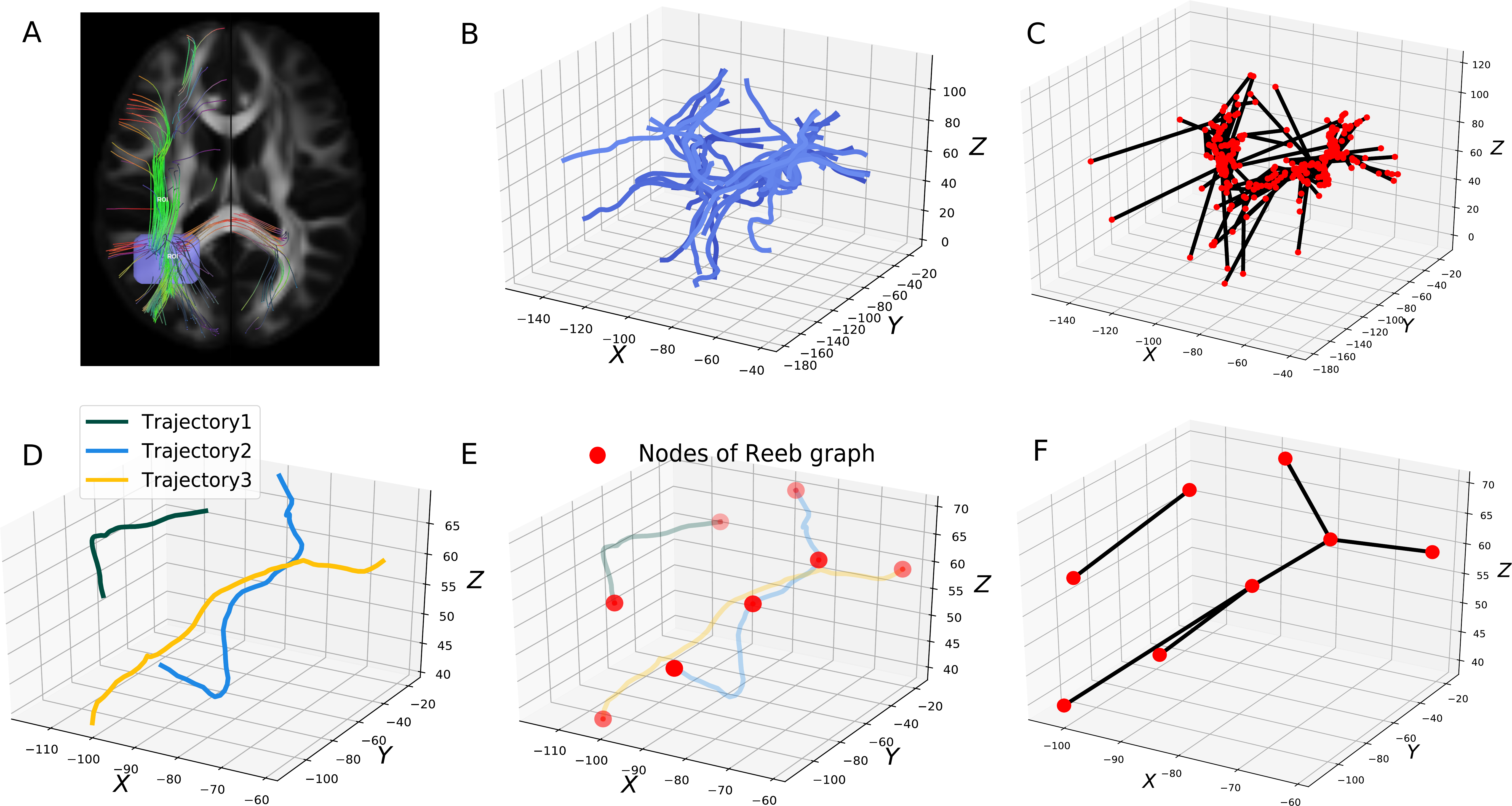}}
        \caption{(A) shows the fiber tracts created by DSI Studio (\protect\url{http://dsi-studio.labsolver.org/Manual/Fiber-Tracking}) for an example ROI. (B) shows the example white matter fiber tracts in 3D. (C) exhibits the corresponding $\mathcal{R}$ for the example fibers. (D) shows three fibers from B to form a qualitative impression of our proposed algorithm. (E) indicates the nodes of $\mathcal{R}$ overlapped on the trajectories. (F) represents the proposed grouping structure with the vertices and edges.}
        \label{fig:real}

    \end{figure}
\begin{theorem}
For a given set of trajectories, $\mathcal{I} = \{T_1, T_2, ..., T_n\}$ with a total of $N$ points, the Reeb graph $\mathcal{R}$ of $\mathcal{I}$ can be computed in $O(N \log N)$ time.
\end{theorem}
\begin{proof}
It is possible to compute the connected components of a graph $G(V^{'}, E^{'})$ with $\abs{V^{'}}$ vertices and $\abs{E^{'}}$ edges using Breadth First Search (BFS) or Depth First Search (DFS) in $O(N^2)$ time. But, since we know all of the points of the given input $\mathcal{I}$ at which any event occurs, we can use a dynamic graph connectivity approach~\cite{parsa2013deterministic} to improve the computation time. This method allows connectivity operations, inserts, and deletes, in $O(log N)$ time. In the worst case, we modify and query the graph $G$ to get the connected components for all the points in $\mathcal{I}$. Hence, the total time required for the construction of $\mathcal{R}$ is $O(NlogN)$.
\end{proof}

\section{Examples and Applications}
\begin{figure}[t]
\scriptsize
            \centering
            \centerline{\includegraphics[scale = 0.26]{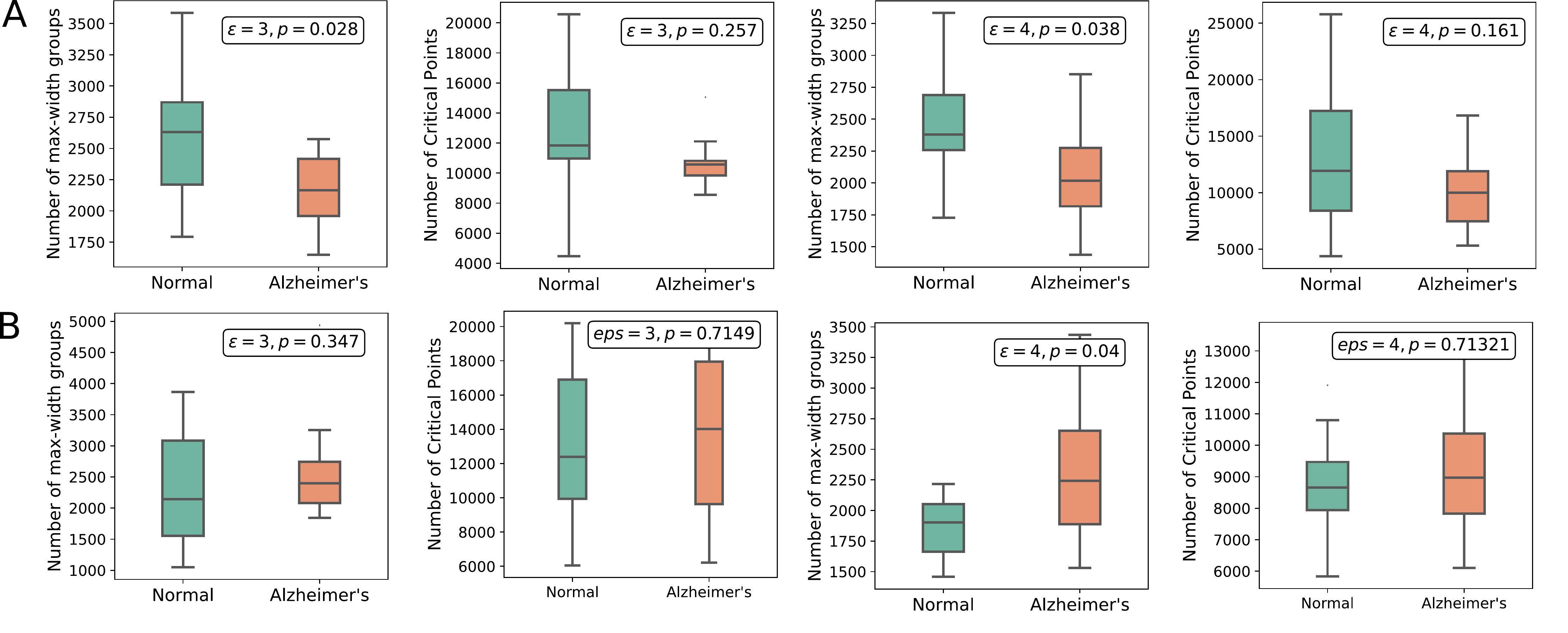}}
        \caption{Statistical analysis results for different values of $\epsilon$ showing the comparison between normal and Alzheimer's subjects across properties of $\mathcal{R}$ for ROI A) Posterior Cingulate Gyrus and B) Middle Occipital Gyrus }
        \label{fig:box}

    \end{figure}
To the best of our knowledge, there are no existing modeling methods in the literature for brain fibers that can be used to compare our method directly.
To provide the proof of concept and demonstrate utility, we evaluate our proposed algorithm on real data and validate manually as illustrated in Fig.~\ref{fig:real}. To design a case study demonstrating the utility of $\mathcal{R}$, we randomly select 22 subjects (11 Normal and 11 Alzheimer's patient) from the publicly available Alzheimer's Disease Neuroimaging Initiative (ADNI)~\cite{Petersen2009} dataset (\url{http://adni.loni.usc.edu/}). We evaluate the qualitative representation of critical points using our model on fibers for random ROIs. All the analytically significant points are captured by $\mathcal{R}$ through nodes and edges, which are highly consistent across subjects.
The proposed model can be employed in the existing deep learning and machine learning algorithms to provide new insights into the structure and the function of the brain. Similar to recent research works where graph theory-based features are utilized for classification tasks, we compute the total number of max-width $\epsilon$-connected groups that is $|E|$ and the aggregate of significant points on fibers that is $|V|$. We also calculate network properties such as clustering, centrality, modularity, and efficiency of $\mathcal{R}$. We choose two ROIs: Posterior Cingulate Gyrus and Middle Occipital Gyrus from the left hemisphere based on the Automated Anatomical Labelling (AAL) atlas~\cite{Rondina2018,TzourioMazoyer2002} and compute tractography consisting of 1000 fibers in each ROI for each subject. We used Q-Space Diffeomorphic Reconstruction as implemented in DSI Studio \cite{Yeh2011} to compute the fibers. In Fig.~\ref{fig:box}, we show the distribution of a set of properties that can be used to facilitate comparisons between Alzheimer's and normal subjects. By comparing the p-value for the ROIs shown in Fig.~\ref{fig:box}, we can conclude that Posterior Cingulate Gyrus (lesser p-value) is a more significant ROI than Middle Occipital Gyrus. This is in accordance with the study~\cite{Rondina2018} that highlights the relevant ROIs for Alzheimer's disease. The average run time of our implementation for examples consisting of 132,000 points on average was 42 seconds on Intel Core CPU 4GHz processor with 32 GB RAM. 
\section{Conclusion}
Our paper proposes the study of the spatial evolution of neuronal trajectories including the algorithmic analysis. We also demonstrate how our proposed reduced graph encodes the critical points of the pathways. Point correspondence of the critical coordinates in the 3D brain space calculated in our algorithm is an essential requirement of the tract-orientated quantitative analysis which is overlooked in the previous works. This aids in localizing and underpinning the points of interest in white matter tracts. Through our preliminary experiments, we show a set of properties of the Reeb graph that can be used to distinguish between the Alzheimer's patients and control subjects. In future, we plan to utilize graph-theoretic concepts to analyze the Reeb graph models of white matter fibers. We intend to further evaluate the reproducibility of our approach on additional datasets. Integrating graph-theoretic features of the Reeb graph with data-driven learning approaches can greatly improve our understanding of various human disease pathways.
\bibliographystyle{splncs04}
\bibliography{mybibliography}
%




\end{document}